%% file: main.tex
\documentclass{ecai} 

\usepackage{latexsym}
\usepackage{amssymb}
\usepackage{amsmath}
\usepackage{amsthm}
\usepackage{booktabs}
\usepackage{enumitem}
\usepackage{graphicx}
\usepackage{color}

\usepackage[dvipsnames]{xcolor}
\usepackage{times}
\usepackage{svg}
\usepackage{soul}
\usepackage{url}
\usepackage[hidelinks]{hyperref}
\usepackage[utf8]{inputenc}
\usepackage[small]{caption}
\usepackage{graphicx}
\usepackage{amsmath,amssymb,dsfont}
\usepackage{amsthm}
\usepackage{booktabs}
\usepackage{algorithm}
\usepackage{algorithmic}
\usepackage[switch]{lineno}
\usepackage{enumitem}
\usepackage{mathtools}

\newtheorem{theorem}{Theorem}

\newtheorem{corollary}[theorem]{Corollary}
\newtheorem{proposition}[theorem]{Proposition}

\newtheorem{example}{Example}
\newtheorem{definition}{Definition}

\newcommand{\BibTeX}{B\kern-.05em{\sc i\kern-.025em b}\kern-.08em\TeX}

\input{macros}

\begin{document}


\begin{frontmatter}


\paperid{123} 


\title{The Distributional Uncertainty of the SHAP Score in Explainable Machine Learning}


\author[A,D]{\fnms{Santiago}~\snm{Cifuentes}\orcid{0000-0002-6375-2045}\thanks{Corresponding Author. Email: scifuentes@dc.uba.ar}}
\author[B]{\fnms{Leopoldo}~\snm{Bertossi}\orcid{0000-0002-1144-3179}\footnote{Prof. Emeritus, Carleton Univ., Ottawa, Canada. \ bertossi@scs.carleton.ca.\\ Senior Researcher IMFD, Chile.}}
\author[C,A]{\fnms{Nina}~\snm{Pardal}\orcid{0000-0002-5150-6947}} 
\author[A,D]{\fnms{Sergio}~\snm{Abriola}\orcid{0000-0002-1979-5443}}
\author[E]{\fnms{Maria Vanina}~\snm{Martinez}\orcid{https://orcid.org/0000-0003-2819-4735}}
\author[F]{\fnms{Miguel}~\snm{Romero}\orcid{0000-0002-2615-6455}}

\address[A]{Instituto de Ciencias de la Computaci\'on, UBA-CONICET, Argentina}
\address[B]{Universidad San Sebasti\'an, FIAD, Santiago, Chile}
\address[C]{Department of Computer Science, University of Sheffield, UK}
\address[D]{Department of Computer Science, Universidad de Buenos Aires, Argentina}
\address[E]{Artificial Intelligence Research Institute (IIIA-CSIC), Spain}
\address[F]{Department of Computer Science, Universidad Cat\'olica de Chile, Chile}
\address[F]{CENIA Chile}


\begin{abstract}
    Attribution scores reflect how important the feature values in an input entity are for the output of a machine learning model. One of the most popular attribution scores is the SHAP score, which is an instantiation of the general Shapley value used in coalition game theory. The definition of this score relies on a probability distribution on the entity population. Since the exact distribution is generally unknown, it needs to be assigned subjectively or be estimated from data, which may lead to misleading feature scores. 
    In this paper, we propose a principled framework for reasoning on SHAP scores under unknown entity population distributions. In our framework, we consider an uncertainty region that contains the potential distributions, and the SHAP score of a feature becomes a function defined over this region. We study the basic problems of finding maxima and minima of this function, which allows us to determine tight ranges for the SHAP scores of all features. In particular, we pinpoint the complexity of these problems, and other related ones, showing them to be intractable. Finally, we present experiments on a real-world dataset, showing that our framework may contribute to a more robust feature scoring.
\end{abstract}

\end{frontmatter}


\section{Introduction}
\label{sec:intro}

Proposing and investigating different forms of explaining and interpreting the outcomes from AI-based systems has become an effervescent area of research and applications \cite{molnar2020interpretable}, leading to the emergence of the area of Explainable Machine Learning. In particular, one wants to explain results obtained from ML-based classification systems. A widespread approach to achieve this consists in assigning numerical {\em attribution scores} to the feature values that, together, represent a given entity under classification, for which a given label has been obtained. The score of a feature value indicates how relevant it is for this output label.  

One of the most popular attribution scores is the so-called \emph{SHAP score} \cite{lundberg2017unified,shapNature}, which is a particular form of the general \emph{Shapley value} used in coalition game theory \cite{shapley1953,roth}. 
SHAP, as every instantiation of the Shapley value, requires a {\em wealth function} shared by all the players in the coalition game. SHAP uses one that is an expected value based on a probability distribution on the entity population.\footnote{Since SHAP's inception, several variations have been proposed and investigated, but they all rely on some probability distribution. In this work we stick to the original formulation.} 
Since the exact distribution is generally unknown, it needs to be assigned subjectively or be estimated from data. This may lead to different kinds of errors, and in particular, to  misleading feature scores. 

In this work, we propose a principled framework for reasoning on SHAP scores under distributional uncertainty, that is, under an unknown distribution over the entity population.
We focus on \emph{binary classifiers}, i.e., classifiers that returns $1$ (accept) or $0$ (reject). We also assume that the inputs to these classifiers are \emph{binary features}, i.e., features that can take values $0$ or $1$. Furthermore, we focus on product distributions. Their use for SHAP computation is common, and imposes feature independence \cite{arenas-et-al:jmlr23,vandenbroeck-et-al:guy22}. In practice, one frequently uses an {\em empirical product distribution} (of the empirical marginals), which may vary depending on the data set from which the sampling is performed \cite{bertossi-et-al:deem}. 
We see our concentration on product distributions as an important first step towards the distributional analysis of SHAP.  \ignore{The idea behind our framework is to explicitly consider an \emph{uncertainty region} that contains the estimated distributions.}


\ignore{+++
\miguel{agregar referencias y escribir algo defendiendo el choice de product distribution}

 \miguel{explicar que tomamos hyperectangles como regiones y defender la decision; en el caso de distribuciones producto es lo que mas tiene sentido, y sigue siendo bastante expresivo}
 }

Our approach allows us to reinterpret and analyze SHAP as a \emph{function} defined on the uncertainty region. As it turns out, this function is always a polynomial on $n$ variables (where $n$ is the number of features), and hence we refer to it as the \emph{SHAP polynomial}. 
We can then analyze the behavior of this polynomial to gain concrete insights on the importance of a feature.
\begin{example}\label{ex:running_example}
Consider the classifier $M$ given in Table \ref{tab:running_example}. Let $e$ be the null entity (first row), and assume a product distribution $\langle p_x,p_y,p_z\rangle$ over the feature space, e.g., $\prob(x=1,y=0,z=1) = p_x (1-p_y)p_z$. The SHAP score for entity $e$ and feature $z$ depends on the probabilities $p_x$, $p_y$ and $p_z$, and this relation can be expressed through the following function:
\begin{eqnarray}
    \shap(\aModel,\anEntity,z) &=& \shap_{\aModel,\anEntity,z}(p_x, p_y, p_z) \nonumber \\
    &=& \frac{1}{6} p_z(-4p_xp_y + 3p_x + 3p_y). \label{eq:poly}
    \end{eqnarray}
We call this function the SHAP polynomial for entity $e$ and feature $z$ (details are provided in Section~\ref{sec:preliminaries}). Observe that the term $(-4p_xp_y + 3p_x + 3p_y)$ is strictly positive whenever $p_x, p_y \neq 0$, and in those cases the SHAP score for $z$ grows when $p_z$ grows. This is intuitive: as $p_z$ grows the probability that $e'(z) = 1$ for a randomly chosen entity $e'$ increases and this predisposes the classifier towards rejection (three out of four entities with $z=1$ are rejected by $M$). Therefore, the fact $e(z) = 0$ becomes more informative, and consequently the SHAP score increases. Meanwhile, if $p_x = p_y = 0$ then the prediction is $1$ with probability 1 independently of the value of $z$, and the SHAP score of $z$ is 0.
\begin{table}[h]
    \centering
    \begin{tabular}{ccc|c}
        $x$ & $y$ & $z$ & $M$ \\
        \hline
        0 & 0 & 0 & \textbf{1} \\
        0 & 0 & 1 & \textbf{1} \\
        0 & 1 & 0 & \textbf{1} \\
        1 & 0 & 0 & \textbf{1} \\
    \end{tabular}
    \caption{Classifier $M$, it labels the remaining entities with \textbf{0}.}
    \label{tab:running_example}
\end{table}
\end{example}

There are different kinds of analysis one can carry out on the SHAP polynomial. As a first step, in this work we investigate the basic problem of finding maxima and minima of SHAP scores in the given region. This allows us to compute \textit{SHAP intervals} for each feature: a range of all the values that the SHAP score can attain in the uncertainty region. 
We believe these tight ranges to be a valuable tool for reasoning about feature importance  under distributional uncertainty. For instance, the length of the interval for a given feature provides information about the robustness of SHAP for that feature. 
Furthermore, changes of sign in a SHAP interval tells us if and when a feature has negative or positive impact on classifications. A global analysis of SHAP intervals can also be used to {\em rank features} according to their general importance.   

To determine the SHAP intervals it is necessary to find minimal upper-bounds and maximal lower-bounds for the SHAP score in the uncertainty region. Formulated in terms of thresholds, these problems turn out to be in the class \NP{}\footnote{See \cite{garey1979computers} for a standard introduction into the complexity classes considered in this paper.} for a wide class of classifiers. Furthermore, we establish that this problem becomes \NPcomplete even for simple models such as {\em decision trees} (and other classifiers that share some properties with them). Notice that computing SHAP for decision trees can be done in polynomial time under the product and uniform distributions. Actually, this result can be obtained for larger classes of classifiers that include decision trees~\cite{arenas-et-al:jmlr23,vandenbroeck-et-al:guy22}.

We also propose and study three other problems related to the behavior of the SHAP score in the uncertainty region, and obtain the same complexity theoretical results as for the problem of computing the maximum and minimum SHAP score. These problems are: (1) deciding whether there are two different distributions in the uncertainty region such that the SHAP score is positive in one of them and negative in the other one (and therefore there is no certainty on whether the feature contributes positively or negatively to the prediction), (2) deciding if there is some distribution such that the SHAP score is 0 (i.e., if the feature can be considered \textit{irrelevant} in some sense), and (3) deciding if for every distribution in the uncertainty region it holds that a feature $x$ is better ranked than a feature $y$ (i.e., if $x$ \textit{dominates} $y$). 

We remark that the upper bound of \NP{} for all these problems is not evident since they all involve reasoning around polynomial expressions, and in principle we may not have polynomial bounds for the size of the witnesses. Moreover, as we will see further on, the SHAP polynomial cannot even be computed explicitly for most models.

To conclude, we carry out an experimentation to compute these SHAP intervals over a real dataset in order to observe what additional information is provided by the use of the SHAP intervals. 
We find out that, under the presence of uncertainty, most of the rankings are \textit{sensitive} to the choice of the distribution over the uncertainty region: the ranking may vary depending on the chosen distribution, even when taking into account only the top 3 ranked features. We also study how this sensitivity decreases as the precision of the distribution estimation increases.

\paragraph{Related work.} 
Close to our work, but aiming towards a different direction, we find the problem of {\em distributional shifts} \cite{carvahlo-et-al:carvalho2019machine,molnar2020interpretable,linardatos-et-al:linardatos2020explainable}, which in ML  occur when the distribution of the training data differs from the data the classifier  encounters in a particular application. This discrepancy poses significant challenges, as can lead to decreased performance and unexpected behavior of models.

Also related is the problem of score {\em robustness} \cite{robust,marques-silva}: one can analyze how scores change under small perturbations of an input. In our case, we study uncertainty at the level of the underlying probability distributions.



The work \cite{slack2021reliable} also tries to address uncertainty in the importance of features for local explanations, but does so from a Bayesian perspective: they use a novel sampling procedure to estimate credible intervals around the mean of the feature importance, and derive closed-form expressions for the number of perturbations required to reach explanations within the desired levels of confidence. 

Finding optimal intervals under uncertainty as done here, is reminiscent of finding tight ranges for aggregate queries from uncertain databases which are repaired due to violations of integrity constraints~\cite{arenas-et-al:aggr}.

Finally, other lines of work such as~\cite{merrick2020explanation} aim to understand the uncertainty that arises from approximation errors when computing the Shapley values via a sampling procedure over the feature space. In such contexts, the distribution is usually assumed as given, and therefore these works focus on formalizing a scenario that differs from ours. Moreover, all our analyses and algorithms are based on optimal confidence intervals that arise from exact computation of the Shapley values.


\paragraph{Our contributions.}
In this work we make the following contributions:
\begin{enumerate}
    \item We propose a new approach to understand the SHAP score by interpreting it as a polynomial evaluated over an uncertainty region of probability distributions.
    \item We analyze at which points of the uncertainty region the maximum and minimum values for SHAP are attained. 
    \item We establish \NP-completeness of deciding if the score of a feature can be larger than a given threshold; we also show \NP-completeness for some related problems. 
    \item We provide experimental results showing how SHAP scores can vary over the uncertainty region, and how considering uncertainty makes it possible to define more nuanced rankings of feature importance.
\end{enumerate}

\paragraph{Organization.}
This paper is structured as follows: \ In Section~\ref{sec:preliminaries} we introduce notation and recall basic definitions. In Section~\ref{sec:framework}, we formalize our problems and obtain the first results. Section~\ref{sec:complexity} presents our main complexity results. In Section~\ref{sec:experiments}, we describe our experiments and show their outcome. Finally, in Section~\ref{sec:conclusions}, we make some final remarks and point to open problems. Proofs for all our results can be found in  \cite{cifuentes2024distributional}.

\section{Preliminaries}
\label{sec:preliminaries}

Let $X$ be a finite set of \emph{features}. An \emph{entity} $e$ over $X$ is a mapping $e:X \to \{0,1\}$.
We denote by $\entities(X)$ the set of all entities over $X$. 
Given a subset of features $S \subseteq X$ and an entity $e$ over $X$, we define the set of entities \emph{consistent with} $e$ on $S$ as: 
\begin{align*}
\consistsWith(e, S) \defeq \{e' \in \entities(X): e'(x)=e(x) \text{ for all } x \in S\}. 
\end{align*}

As already discussed in Section~\ref{sec:intro}, we shall consider \emph{product distributions} as our basic probability distributions over the entity population $\entities(X)$. A product distribution $\prob$ over $\entities(X)$ is parameterized by values $(p_x)_{x\in X}$. For every $e\in \entities(X)$ we have:
\begin{align*}
    \prob(e) = \prod_{x\in X : e(x)=1} p_x \prod_{x \in X : e(x)=0} (1-p_x).
\end{align*}
That is, each feature value $e(x)$ is chosen independently with a probability according to $p_x$ ($e(x)=1$ with probability $p_x$).

A \emph{(binary) classifier or model} $M$ over $X$ is a mapping $M:\entities(X) \to \{0,1\}$. We say that $M$ \emph{accepts} $e$ if $M(e) = 1$, otherwise $M$ \emph{rejects} the entity.
Let $M$ be a binary classifier and $e$ an entity, both over $X$. We define the function $\phi_{M,e}:2^{X}\to [0,1]$ as:
\begin{align*}
   \phi_{M,e}(S)\defeq \expectancy[M\, |\,  \consistsWith(e, S)].
\end{align*}
In other words, $\phi_{M,e}(S)$ is the expected value of $M$ conditioned to the event $\consistsWith(e, S)$. More explicitly:
\begin{align*}
   \phi_{M,e}(S) = \sum_{e' \in \consistsWith(e, S)} \prob(e'\, |\, \consistsWith(e, S)) M(e').
\end{align*}
A direct calculation shows that the conditional probability $\prob(e'\, |\, \consistsWith(e, S))$ can be written as:

\begin{align*}
    \prob(e'\, |\, \consistsWith(e, S)) = \prod_{x \in X \setminus S : e'(x)=1} p_x \prod_{x \in X \setminus S : e'(x) = 0} (1-p_x).
\end{align*}

\medskip

The function $\phi_{M,e}$ can be used as the wealth function in the general formula of the Shapley value~\cite{shapley1953,roth} to obtain the SHAP score of the feature values in $e$.

\begin{definition}[SHAP score]
    Given a classifier $M$ over a set of features $X$, an entity $e$ over $X$, and a feature $x \in X$, the \emph{SHAP score} of feature $x$ with respect to $M$ and $e$ is 
    
\begin{align*}
    \shap(M,e,x) \defeq \sum_{S \subseteq X \setminus x} c_{|S|} \left(\phi_{M,e}(S \cup \{x\}) - \phi_{M,e}(S)\right),
\end{align*}

where $c_{i} \defeq \frac{i!(|X| - i - 1)!}{|X|!}$. 
\end{definition}

Intuitively, the SHAP score intends to measure how the inclusion of $x$ affects the conditional expectation of the prediction. In order to do this, it considers every possible subset $S\subseteq X \setminus \{x\}$ of the features and compares the expectation for the set $S$ against $S \cup \{x\}$. A  score close to $1$ implies that $x$ heavily leans the classifier $M$ towards acceptance, while a score close to $-1$ indicates that it leans the prediction towards rejection (note that SHAP always takes values in $[-1,1]$).

\begin{example}
    Consider again the model $M$ from Table~\ref{tab:running_example}. It can be shown that if $p_x = p_y = \frac{1}{2}$ and $p_z = \frac{3}{4}$ then feature $z$ has SHAP score 0.25 while $x$ and $y$ have score 0.1875. Meanwhile, if $p_z=\frac{1}{4}$ then $Shap(M,e,z) \sim 0.08$ and $Shap(M,e,x) = Shap(M,e,y) \sim 0.15$. 
\end{example}

In practical applications, the exact distribution of entities is generally unknown and subjectively assumed or estimated from data. 
The previous example shows that the choice of the underlying distribution can have severe effects when establishing the importance of features in the classifications. To overcome these problems, in the next section we  formalize the notion of distributional uncertainty and
present our framework for reasoning about SHAP scores in that setting.

\vspace{-2mm}
\section{SHAP under Distributional Uncertainty}
\label{sec:framework}


The general idea of our framework is as follows: we explicitly consider a set that contains the  potential distributions. This provides us with what we call the \emph{uncertainty region}. This allows us to  reinterpret the SHAP score as a \emph{function} from the uncertainty region to $\mathbb{R}$. We can then analyze the behavior of this function in order to gain concrete insights about the importance of a feature.

Recall that a product distribution is determined by its parameters $(p_x)_{x\in X}$. For convenience, we always assume an implicit ordering on the features. Hence, we can identify our space of probability distributions over $\entities(X)$ with the set $[0,1]^{|X|}$. In order to define uncertainty regions, it is natural then to consider \emph{hyperrectangles} $\anInterval \subseteq [0,1]^{|X|}$, i.e., subsets of the form $\anInterval=\bigtimes_{x\in X} [a_x, b_x]$. Intuitively, these regions correspond to independently choosing a confidence interval $[a_x, b_x]$ for the unknown probability $p_x$, for each feature $x\in X$.

\begin{example}
Within the setting from  Example \ref{ex:running_example}, consider the following uncertainty regions defined by hyperrectangles $\anInterval_1$ and $\anInterval_2$, respectively:
\[
\anInterval_1 := [\frac{1}{3}, \frac{1}{2}] \times [1, 1] \times [\frac{1}{3}, \frac{2}{3}]
\]
\[
\anInterval_2 := [\frac{1}{2}, \frac{1}{2}] \times [\frac{1}{2}, 1] \times [0, \frac{1}{2}]
\]
Notice that the SHAP polynomial in region 1 attains a maximum at $p_x=\frac{1}{3}$, $p_y=1$, and $p_z=\frac{2}{3}$, where the maximum score is $\frac{8}{27}$. The minimum value, corresponding to the score $\frac{5}{36}$, is attained at $p_x=\frac{1}{2}$, $p_y=1$, $p_z=\frac{1}{3}$.

For the second region, the maximum is attained at $p_x=\frac{1}{2}$, $p_y=1$, and $p_z=\frac{1}{2}$, and the maximum value is $\frac{5}{24}$. Similarly, the minimum score $0$ is attained whenever $p_z=0$.
\end{example}

The SHAP score now becomes a function taking probability distributions and returning real values. This is formalized below.

\begin{definition}[SHAP polynomial]
 Given a classifier $M$ over a set of features $X$, an entity $e$ over $X$, and a feature $x \in X$, the \emph{SHAP polynomial} $\shap_{M,e,x}$  is the function from $[0,1]^X$ to $\mathbb{R}$ mapping each $(p_x)_{x\in X}$ to the SHAP score $\shap(M,e,x)$ using $(p_x)_{x\in X}$ as the underlying product distribution. 
\end{definition}

As the name suggests, the SHAP polynomial $\shap_{M,e,x}$ is actually a multivariate polynomial on the variables $(p_x)_{x\in X}$. Moreover, it is a \emph{multilinear} polynomial: it is linear on each of its variables separately (equivalently, no variable occurs at a power of $2$ or higher).

\begin{proposition}
Given a classifier $M$ over a set of features $X$, an entity $e$ over $X$, and a feature $x \in X$, the \emph{SHAP polynomial} $\shap_{M,e,x}$ is a multilinear polynomial on variables $(p_x)_{x\in X}$.
\end{proposition}

\begin{proof}
Note that $\phi_{M,e}(S)$ is a multilinear polynomial for any subset of features $S\subseteq X$. Since $Shap_{M, e, x}$ is a weighted sum of expressions of the form $\phi_{M, e}(S)$ the result follows by observing that multilinear polynomials are closed with respect to sum and product by constants.
\end{proof}

We can then reason about the importance of features via the analysis of SHAP polynomials. Here we concentrate on the fundamental problems of finding maxima and minima of these polynomials over the uncertainty region. This allows us to determine the \emph{SHAP interval} of a feature, i.e., the set of possible SHAP scores of the feature over the uncertainty region. SHAP intervals may provide useful insights. For instance, obtaining smaller SHAP intervals for a feature suggests its SHAP score is more robust against uncertainty. On the other hand, they can be used to assess the relative importance of features (see Section~\ref{sec:experiments} for more details).

We aim to characterize the complexity of these problems, thus we formulate them as decision problems\footnote{The encoding of $M$ depends on the class of classifiers considered, while $\anInterval$ is given by listing the rationals $a_i, b_i$ ($1\leq i\leq n$).}:
\begin{center}
\fbox{\begin{minipage}{20em}
  \textsc{Problem}: \problemMaxShap
  
\textsc{Input}: A classifier $M$, an entity $e$, a feature $x$, a hyperrectangle  $\anInterval$ and a rational number $\aBound$.

\textsc{Output}: Is there a point $\textbf{p} \in \anInterval$ such that $Shap_{M,e,x}(\textbf{p}) \geq \aBound$?.
\end{minipage}}
\end{center}
The problem \problemMinShap{} is defined analogously by requiring $Shap_{M,e,x}(\textbf{p}) \leq \aBound$. 
We also consider some related problems: 
\begin{itemize}
    \item \problemAmbiguity{}: given a classifier $M$, an entity $e$, a feature $x$, and a hyperrectangle  $\anInterval$, check whether there are two points $\textbf{p}_1, \textbf{p}_2  \in \anInterval$ such that $Shap_{M,e,x}(\textbf{p}_1) > 0$ and $Shap_{M,e,x}(\textbf{p}_2) < 0$. This problem can be understood as a simpler test for robustness (in comparison to actually computing the SHAP intervals).
    \item \problemIrrelevancy{}:
    given a classifier $M$, an entity $e$, a feature $x$, and a hyperrectangle  $\anInterval$, check whether there is a point $\textbf{p} \in \anInterval$ such that $Shap_{M,e,x}(\textbf{p}) = 0$. This is the natural adaptation of checking \emph{irrelevancy} of a feature (score equal to $0$) to the uncertainty setting.
    \item \problemDominance{}: 
    given a classifier $M$, an entity $e$, features $x$ and $y$, and a hyperrectangle  $\anInterval$, check whether $x$ \emph{dominates} $y$, that is, for all points $\textbf{p} \in \anInterval$, we have $Shap_{M,e,x}(\textbf{p}) \geq Shap_{M,e,y}(\textbf{p})$. The notion of dominance provides a safe way to compare features under uncertainty.
\end{itemize}


\vspace{-3mm}
\section{Complexity Results}
\label{sec:complexity}

We now present our main technical contributions, namely, we pinpoint the complexity of the problems presented in the previous section. 

\subsection{Preliminaries on multilinear polynomials}
\label{sec:prelims-multinlinear}

A \emph{vertex} of a hyperrectangle $\anInterval=\times_{i=1}^n [a_i,b_i]$ is a point $\textbf{p}\in \anInterval$ such that $\textbf{p}_i\in \{a_i,b_i\}$ for each $1\leq i\leq n$. The following is a well-known fact about multilinear polynomials (see e.g., ~\cite{laneve-et-al:laneve2010interval}).
For completeness we provide a simple self-contained proof in Appendix~\ref{sec:app-prelims-multilinear}. 

\begin{proposition}
\label{prop:general-maximizesInVertices}
Let $f$ be a multilinear polynomial over $n$ variables. Let $\anInterval\subseteq [0,1]^n$ be a hyperrectangle.
Then the maximum and minimum of $f$ restricted to $\anInterval$ is attained in the vertices of $\anInterval$.
\end{proposition}

Proposition~\ref{prop:general-maximizesInVertices} yields two algorithmic consequences. On the one hand, it induces an algorithm to find the maximum of $f$ over $\anInterval$ in time $2^n poly(|f|)$: simply evaluate the polynomial on all the vertices, and keep the maximum\footnote{We assume $f$ is given by listing its non-zero coefficients and their corresponding monomials.}. 
On the other hand, it shows that this problem is certifiable: to decide whether $f$ can reach a value as big as $\aBound$, we just need to guess the corresponding vertex and evaluate it. We show that, within the usual complexity theoretical assumptions, there is no polynomial algorithm to find this maximum (see Appendix~\ref{sec:app-prelims-multilinear}): 

\begin{theorem}\label{teo:npcompleteMaximizeMultilinear}
    Given a multilinear polynomial $\aMultilinearPolynomial$, a hyperrectangle $\anInterval = \bigtimes_{i=1}^n [a_i, b_i]$, and a rational $\aBound$, the problem of deciding whether there is an $\textbf{x} \in \anInterval$ such that $\aMultilinearPolynomial(\textbf{x}) \geq \aBound$ is \NPcomplete.
\end{theorem}

\subsection{Complexity of \problemMaxShap{}}
\label{sec:complexity-region-max-shap}

It is well-known that computing SHAP scores is hard for general classifiers. For instance, the problem is already \sharpPhard{} when considering \emph{Boolean circuits} \cite{arenas-et-al:jmlr23} or \emph{logistic regression models} \cite{vandenbroeck-et-al:guy22}. For {\em linear perceptrons}, model counting is intractable \cite{barcelo-et-al:neurips20} and by the results in \cite{arenas-et-al:jmlr23}, it follows that SHAP computation for perceptrons is also intractable. \ 
On the other hand, computing maxima of SHAP polynomials for a certain class of classifiers is as hard as computing SHAP scores for that class of classifiers:
if we consider the hyperrectangle consisting of a single point $\anInterval = \bigtimes_{i=1}^n [p_i,p_i]$, the maximum of the SHAP polynomial coincides with the SHAP score for the product distribution $(p_i)_{1\leq i\leq n}$. 
Therefore, we focus on family of classifiers where the SHAP score can be computed in polynomial time. A prominent example is the class of \emph{decomposable and deterministic Boolean circuits}, whose tractable SHAP score computation has been shown recently~\cite{arenas-et-al:jmlr23}. This class contains as a special case the well-known class of \emph{decision trees}. For a formal definition see ~\cite{darwiche2001,DM2002}.
As a consequence of Proposition~\ref{prop:general-maximizesInVertices}, we obtain the following complexity upper bound:

\begin{corollary}
\label{coro:np-region-max-shap}
    Let $\cal F$ be a class of classifiers for which computing the SHAP score for given product distributions can be done in polynomial time. Then \problemMaxShap{} is in \NP{} for the class $\cal F$. In particular, \problemMaxShap{} is in \NP{} for decomposable and deterministic Boolean circuits.
\end{corollary}

Next, we show a matching lower bound for \problemMaxShap{}. Interestingly, this holds even for decision trees. We stress that the \NP-hardness of \problemMaxShap{} does not follow directly from Theorem~\ref{teo:npcompleteMaximizeMultilinear}: in \problemMaxShap{} the multilinear polynomial is given implicitly, and it is by no means obvious how to encode the multilinear polynomials used in the hardness argument of Theorem~\ref{teo:npcompleteMaximizeMultilinear}. Instead, we follow a different direction and encode directly the classical problem of \problemVertexCover.

\begin{theorem}\label{teo:npcomplete_decisiontrees}
    The problem \problemMaxShap{} is \NPhard{} for decomposable and deterministic Boolean circuits. The result holds even when restricted to decision trees.
\end{theorem}

\begin{proof}[Sketch of the proof]
We reduce from the well-known \NPcomplete problem \problemVertexCover: given a graph $G=(V,E)$ and $k\geq 1$, decide whether there is a vertex cover\footnote{Recall a vertex cover of $G=(V,E)$ is a subset of the nodes $C\subseteq V$ such that for each edge $\{u,v\}\in E$, either $u\in C$ or $v\in C$. } of size at most $k$. 

The hardness proof relies on two observations. Firstly, by using $|V|$ features and choosing the hyperrectangle $\anInterval = [0,1]^{|V|}$ as the uncertainty region, there is a natural bijection $\textbf{p}(C) = \textbf{p}^C \in \anInterval$ between the subsets $C \subseteq V$ and the vertices of $\anInterval$ ($v\in C$ iff $p_v=0$). Secondly, by properly picking the entities accepted by model $M$, the SHAP polynomial will be
\begin{align*}
Shap_{M,e,x}(\textbf{p}^C) =  - \sum_{\{u,v\} \in E} p_{u}p_{v} I_{u,v}
    - T_{n,\ell} 
\end{align*}
where $\ell$ is the size of the subset $C$ and the term $T_{n,\ell}$ grows as $\ell$ grows. The term $I_{u,v}$ is positive and works as a \emph{penalization factor} which is ``activated'' if $p_u = p_v = 1$, i.e., when the edge $uv$ is uncovered. Furthermore, by adding an extra feature $w$ to the model (and consequently, another probability $p_w$) we can make this penalization factor arbitrarily big in relation to the \textit{size factor} $T_{n, \ell}$. 

We choose the bound $\aBound$ to be $-T_{n,k}$. If $C$ is a vertex cover of size $\ell$, then each term $p_{u}p_{v} I_{u,v}$ equals 0 and hence $Shap_{M,e,x}(\textbf{p}^C) = -T_{n,\ell}$. On the other hand, if $C$ is not a vertex cover, then some $p_u p_v I_{u,v}$ is non-zero, and by defining an adequate interval for $p_w$ we can ensure that $- p_v p_v I_{u,v} < \aBound$. Hence, the only way to obtain $Shap_{M,e_0,x_{0}}(\textbf{p}^C)\geq \aBound$ is to pick $C$ to be, in the first place, a vertex cover, and secondly, one of size $\ell\leq k$.  
\end{proof}

Both Corollary \ref{coro:np-region-max-shap}
and Theorem \ref{teo:npcomplete_decisiontrees} also apply to \problemMinShap{} (see Appendix~\ref{sec:app-region-min-shap} for details).

\subsection{Related problems}
\label{sec:complexity-related}

In this section we show some results related to the problems proposed in Section \ref{sec:framework}. As in the case of \problemMaxShap{} they turn out to be \NPcomplete{}, even when restricting the input classifiers to be decision trees.



Again, as a consequence of Proposition~\ref{prop:general-maximizesInVertices} we obtain the \NP{} membership for \problemAmbiguity{} and \problemIrrelevancy{}, and the \coNP{} membership for \problemDominance{}.

\begin{corollary}\label{coro:realed_NP}
    Let $\cal F$ be a class of classifiers for which computing the SHAP score for given product distributions can be done in polynomial time. Then the problems \problemAmbiguity{} and \problemIrrelevancy{} are in \NP for the class $\mathcal{F}$, while the \problemDominance{} is in \coNP{}\footnote{The proof for \problemDominance{} follows by observing that $\textit{diff}(\textbf{p}) = Shap_{M,e,x}(\textbf{p}) - Shap_{M,e,y}(\textbf{p})$ is again a multilinear polynomial.}.
\end{corollary}

The hardness for these problems follows under the same conditions as Theorem~\ref{teo:npcomplete_decisiontrees}.

\begin{theorem}\label{teo:related_hard}
    The problems \problemAmbiguity{} and \problemIrrelevancy{} are \NPhard for decision trees, while \problemDominance{} is \coNPhard.
\end{theorem}

\begin{proof}[Sketch of the proof]
    The proof follows the same techniques as the proof for Theorem~\ref{teo:npcomplete_decisiontrees}, through a reduction from \problemVertexCover{}. For the case of \problemIrrelevancy{} we devise a model $M$ such that
    \begin{align}\label{eq:proof_irrelevancy}
        Shap_{M,e,x}(\textbf{p}^C) = T_{n,k} - \sum_{\{u, v\} \in E} p_u p_v I_{u,v} - T_{n,\ell}
    \end{align}
    where $\ell$ is the size of $C$, $T_{n,\ell}$ corresponds to the \textit{size factor}, and $I_{u,v}$ is the \emph{penalization factor} for uncovered edges. Observe that the first term $T_{n,k}$ does not depend on the set $C$, and consequently $Shap_{M,e,x}(\textbf{p}^C) = 0$ if $C$ is a vertex cover of size $k$. The construction is a bit more complex, and we have to add $2(n-k)$ extra features to those considered in the construction of Theorem~\ref{teo:npcomplete_decisiontrees}.

    The proof for \problemAmbiguity{} is obtained by a slight modification of Equation~\ref{eq:proof_irrelevancy} in order to make the SHAP score positive (instead of 0) if $C$ is a vertex cover of size $k$.

    Finally, for the hardness of \problemDominance{} we prove that \problemAmbiguity{}$\leq_p \overline{\mbox{\problemDominance{}}}$\footnote{If $\Pi$ is a decision problem, we denote its complement as $\overline{\Pi}$.}. This reduction is achieved by adding a ``dummy feature'' $w$ that does not affect the prediction of the model. We prove that its SHAP polynomial is constant and equal to 0, and consequently deciding the ambiguity for a feature $x$ is equivalent to deciding the dominance relation between $x$ and $w$.
\end{proof}

\vspace{-3mm}
\section{Experimentation}\label{sec:experiments}

As a case study, we are going to use the California Housing dataset \cite{California}, a comprehensive collection of housing data within the state of California containing 20,640 samples and eight features\footnote{The code developed for the experimentation can be found at \url{https://git.exactas.uba.ar/scifuentes/fuzzyshapley}.}. Our choice of dataset relies mainly on the fact that this dataset has already been considered in the context of SHAP score computation \cite{bertossi2023efficient} and its size and number of features allow us to compute most of the proposed parameters (as the SHAP polynomial itself, for each feature) in a reasonable time. Nonetheless, we recall that the proposed framework can be applied to any dataset, as long as there is some uncertainty on the real distribution of feature values 
and uncertainty regions can be estimated for it (e.g., via sampling the  data as we do here).

Note that, while there is no reason to expect that the probabilities of each feature are independent of each other in the California Housing dataset, we make this assumption in our framework (which is a common one in the literature \cite{vstrumbelj2014explaining,shrikumar-et-al:shrikumar2017learning}).

\subsection{Objectives}

The purpose of this experiment is to use a real dataset to simulate a situation where we have uncertainty over the proportions of each feature in the dataset. We want to derive suitable hyperrectangles representing the distribution uncertainty where our extended concepts of SHAP scores apply, and compare these scores against the usual, point-like SHAP score, in order to reveal cases where these new hypervolume scores are more informative than the traditional scheme. We expect our proposal to be able to detect features whose ranking is vulnerable to small distribution shifts, and aim to study how such sensitivity starts to vanish as we reduce the uncertainty over the distribution (i.e., as the hyperrectangles get smaller).

\subsection{Methodology}

\paragraph{Preparing dataset}
Our framework is defined for binary features, and therefore we need to binarize the features from the California Housing dataset. Of the eight features, seven are numerical and one is categorical. To binarize each of the numerical ones, we take the average value across all entities and use it as a threshold. We also binarize the target \texttt{median\_house\_value} using the same strategy. The categorical one is \texttt{ocean\_proximity}, which describes how close each entity is to the ocean, with the categories being \{\texttt{INLAND}, \texttt{<1H OCEAN}, \texttt{NEAR OCEAN}, \texttt{NEAR BAY}, \texttt{ISLAND}\}. The \texttt{INLAND} one represents the farthest distance away from the ocean, and we binarize that category to 0, while the other ones are mapped to 1.
Finally, we remove the rows with \texttt{NaN} values.

\paragraph{The model $M$}
We take 70\% of the data for training, and keep 30\% for testing. As a model we employ the \textit{sklearn} \texttt{DecisionTreeClassifier}. For regularization we considered both bounding the depth of the tree by 5 and bounding the minimum samples to split a node by 100. The obtained results were similar for both mechanisms, and therefore we only show here the results for the model regularized by restricting the number of samples required to split a node.
The trained model has 80.0\% accuracy and 76.5\% precision. 

\paragraph{Creating the hyperrectangle}
We assume independence between the different distributions of the features, and ignorance of their true probabilities. Therefore, to obtain an estimation of the probability $p_x$ that a feature $x$ has value $1$ for a random entity, we will sample a subset of the available data and compute an empirical average. We vary the number of samples taken in order to simulate situations with different uncertainty. 

Let $N$ be the number of samples. Given $0 < p < 1$, we sample $\lceil pN \rceil$ entities $5$ times, and for each of these times we compute an empirical average $p_x^j$ for each feature $x$. We then define the estimation for $p_x$ as the median of $\{p_x^j\}_{1 \leq j \leq 5}$ and compute a deviation $\sigma_x$ by taking the deviation over the set $\{p_x^j\}_{1 \leq j \leq 5}$. Finally, the hyperrectangle is defined as $\bigtimes_x [p_x - \sigma_x, p_x + \sigma_x]$.
We experiment with different values for $p$ from $10^{-3}$ to $\frac{1}{2}$.

\paragraph{Comparing SHAP scores}
We pick 20 different entities uniformly at random and compute the SHAP score for each of these entities and for each feature at all the vertices of the different hyperrectangles. Instead of using the polynomial-time algorithm for the SHAP value computation over decision trees, we actually compute an algebraic expression for the SHAP polynomial for each pair of entity and feature, simplify it, and then use it to compute any desired SHAP score given any product distribution.

\vspace{-2mm}
\subsection{Results}

We recall that the \textit{SHAP interval} of feature $x$ for entity $e$ over the hyperrectangle $\anInterval$ is \  $[\min_{\textbf{p} \in \anInterval} Shap_{M,e,x}(\textbf{p}), \max_{\textbf{p} \in \anInterval}Shap_{M,e,x}(\textbf{p})]$. 

In Figure~\ref{fig:shap_intervals} we plot the SHAP intervals of all features for one of the entities we used and two sampling percentages. By observing the intervals it is clear that, even in the presence of the uncertainty of the real distribution, as long as it belongs to the hyperrectangle $\anInterval$ the features \texttt{median\_income} and \texttt{ocean\_proximity} will be the top 2 in the ranking defined by the SHAP score, for both sampling percentages considered. However, when the sampling percentage is small ($p=0.4\%$) the SHAP intervals of both features intersect, and therefore it could be the case that their relative ranking actually depends on the chosen distribution inside $\anInterval$.

To decide whether this happens, one should find two points $\textbf{p}_1,\textbf{p}_2 \in \anInterval$ such that $Shap_{M,e,\texttt{med\_inc}}(\textbf{p}_1) < Shap_{M,e,\texttt{ocean\_prox}}(\textbf{p}_1)$ and $Shap_{M,e,\texttt{med\_inc}}(\textbf{p}_2) > Shap_{M,e,\texttt{ocean\_prox}}(\textbf{p}_2)$ (i.e., solve the \problemDominance{} problem). This can be done by observing that $\textit{diff}(\textbf{p}) = Shap_{M,e,\texttt{med\_inc}}(\textbf{p}) - Shap_{M,e,\texttt{ocean\_prox}}(\textbf{p})$ is a multilinear polynomial, and therefore its maximum and minimum are attained at the vertices of $\anInterval$. By computing $\textit{diff}(\textbf{p})$ on all these points we observed that in 4 of the 256 vertices \texttt{median\_income} has a bigger SHAP score than \texttt{ocean\_proximity}, and consequently their relative ranking depends on the chosen underlying distribution.

\begin{figure}
    \centering
    \includegraphics[scale=0.3]{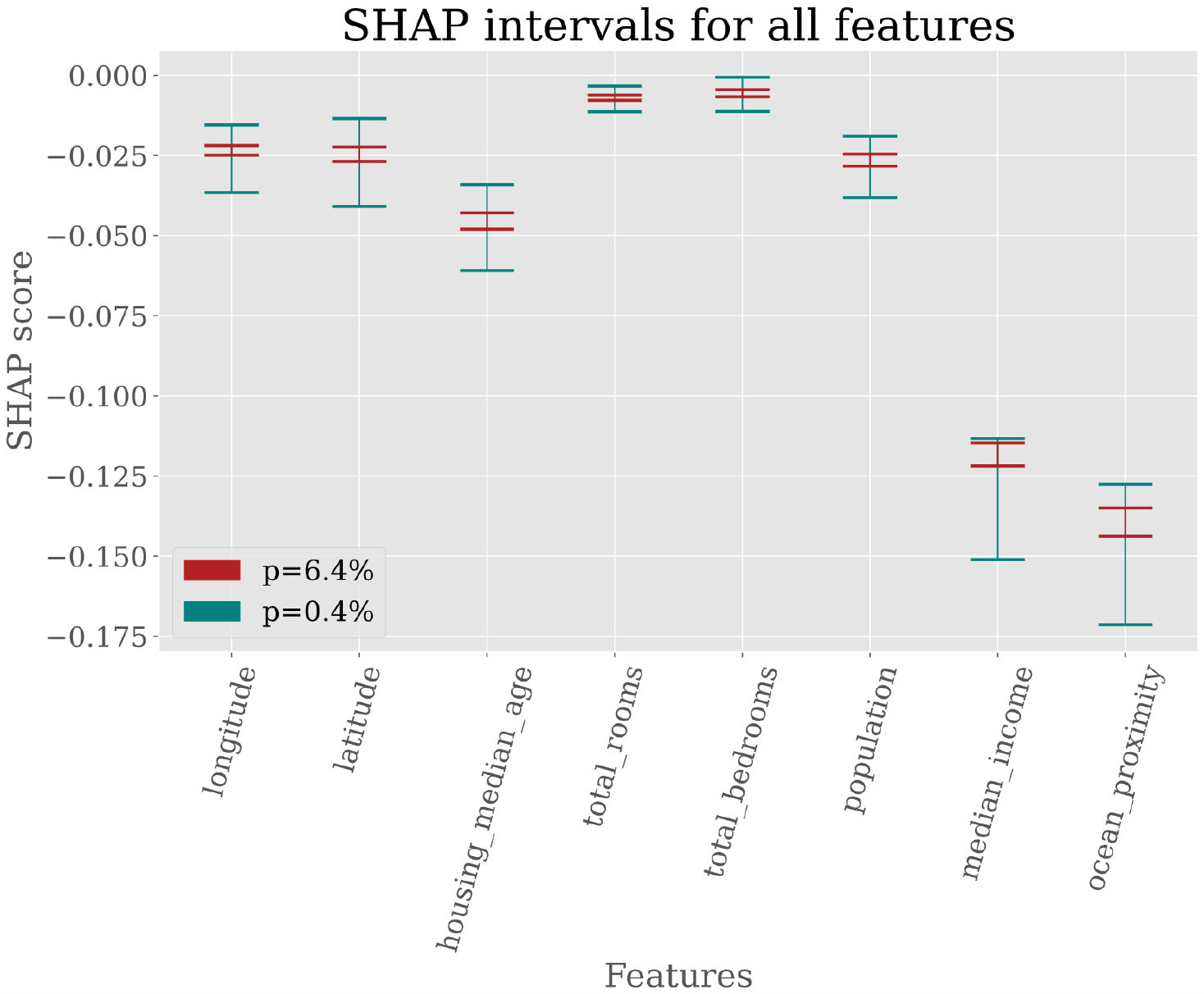}
    \caption{SHAP intervals for all features and a fixed entity, over two different sampling percentages. When $p=0.4\%$ it is clear that, according to the SHAP score, the features \texttt{median\_income} and \texttt{ocean\_proximity} are the two most relevant, but there is an uncertainty on which one of the two is the most important. When the sampling percentage increases to $p=6.4\%$ the SHAP intervals become disjoint, and we can be certain that \texttt{ocean\_proximity} is the most relevant feature. Observe that the same kind of uncertainty exists regarding the third ranked feature.}
    \label{fig:shap_intervals}
    \vspace{0.5cm}
\end{figure}

We can also observe in Figure~\ref{fig:shap_intervals} that something similar happens as well for feature \texttt{housing\_median\_age}. When $p=0.4\%$ its SHAP interval intersects with the intervals of the three features \texttt{longitude}, \texttt{latitude} and \texttt{population}. Nonetheless, by inspecting the difference of the scores at the vertices of $\anInterval$ it can be seen that there is no point in which \texttt{housing\_median\_age} is ranked below the other features (i.e., \texttt{housing\_median\_age} dominates all three features). Meanwhile, if we compare two of the other features such as \texttt{longitude} and \texttt{latitude}, we observe that they have a relevant intersection even when $p=6.4\%$, and in 179 out of the 256 vertices \texttt{latitude} is ranked below \texttt{longitude}.

In Figure~\ref{fig:evolution_regions} we can see how the SHAP intervals shrink as the sampling percentage increases. For all sampling percentages above 6.4\% we can say with certainty that \texttt{ocean\_proximity} is ranked above \texttt{median\_income}. This behavior is natural since the size of the hyperrectangles is getting smaller, because the deviation of the empirical samples $p_x^i$ tends to get smaller as the sampling percentage increases.

\begin{figure}[h!]
    \centering
    \includegraphics[scale=0.3]{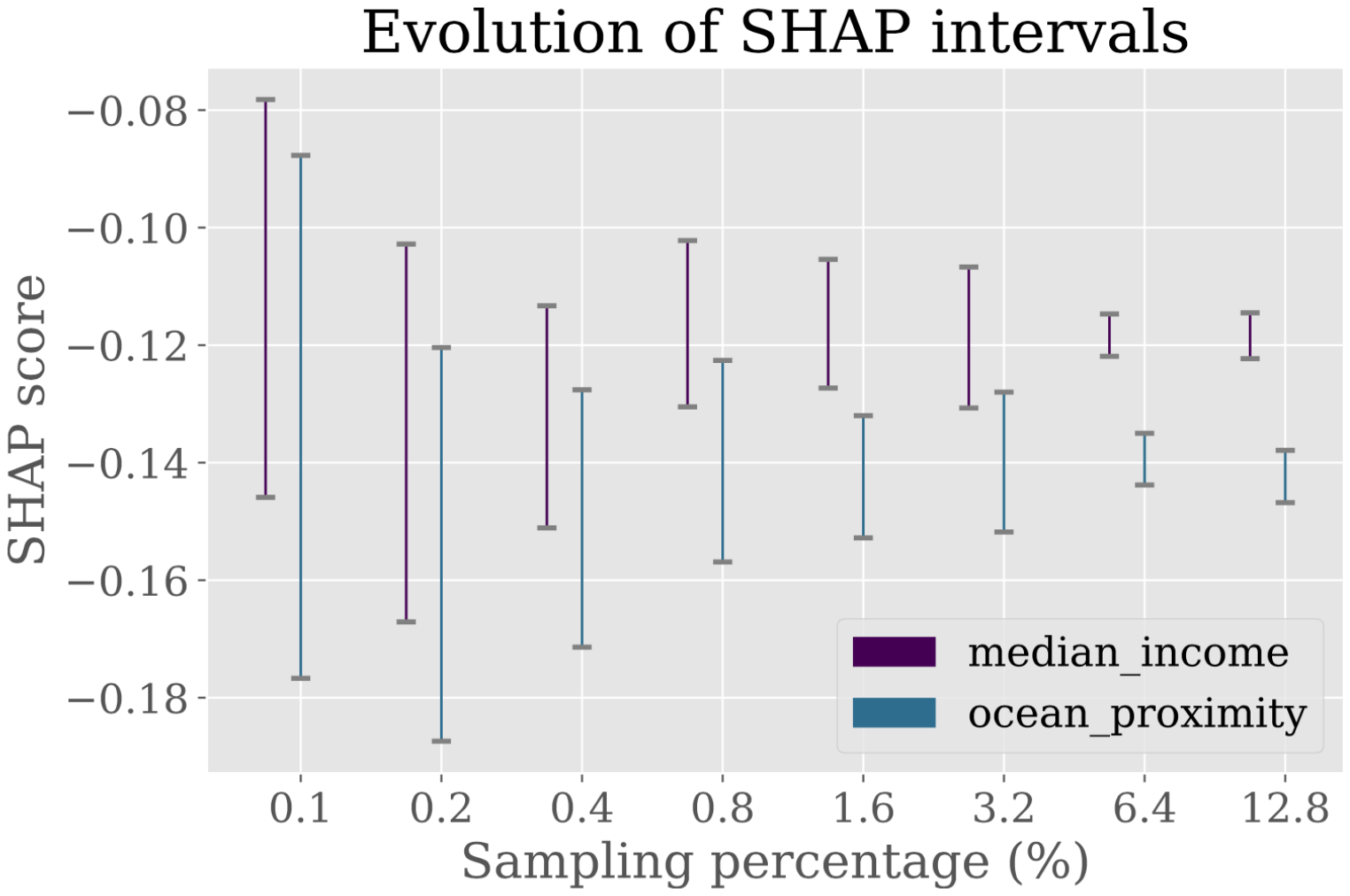}
    \caption{Evolution of the SHAP intervals for features \texttt{median\_income} and \texttt{ocean\_proximity} considering the same entity as in Figure \ref{fig:shap_intervals} and the different sampling percentages.}
    \label{fig:evolution_regions}
    \vspace{0.5cm}
\end{figure}

Finally, in Figure~\ref{fig:rankingSensitivity} we plot the number of entities whose ranking depends on the chosen distribution, for each sampling percentage, and restricting ourselves to some subset of the ranking. We found out that for 10 of the 20 entities the complete ranking defined by the SHAP score depends on the chosen distribution even when the sample percentage is as big as $p=12.8\%$ (recall that due to the way we built our hyperrectangles, we are sampling $5p$ times the dataset, and therefore if $p > 10\%$ we might be inspecting more than half of the data points). If we are only concerned with the top three ranked features, we can observe that even when $p=6.4\%$ there are still 4 entities for which the top 3 ranking is sensitive to the selected distribution. 

\begin{figure}
    \begin{center}
   \includegraphics[scale=0.28]{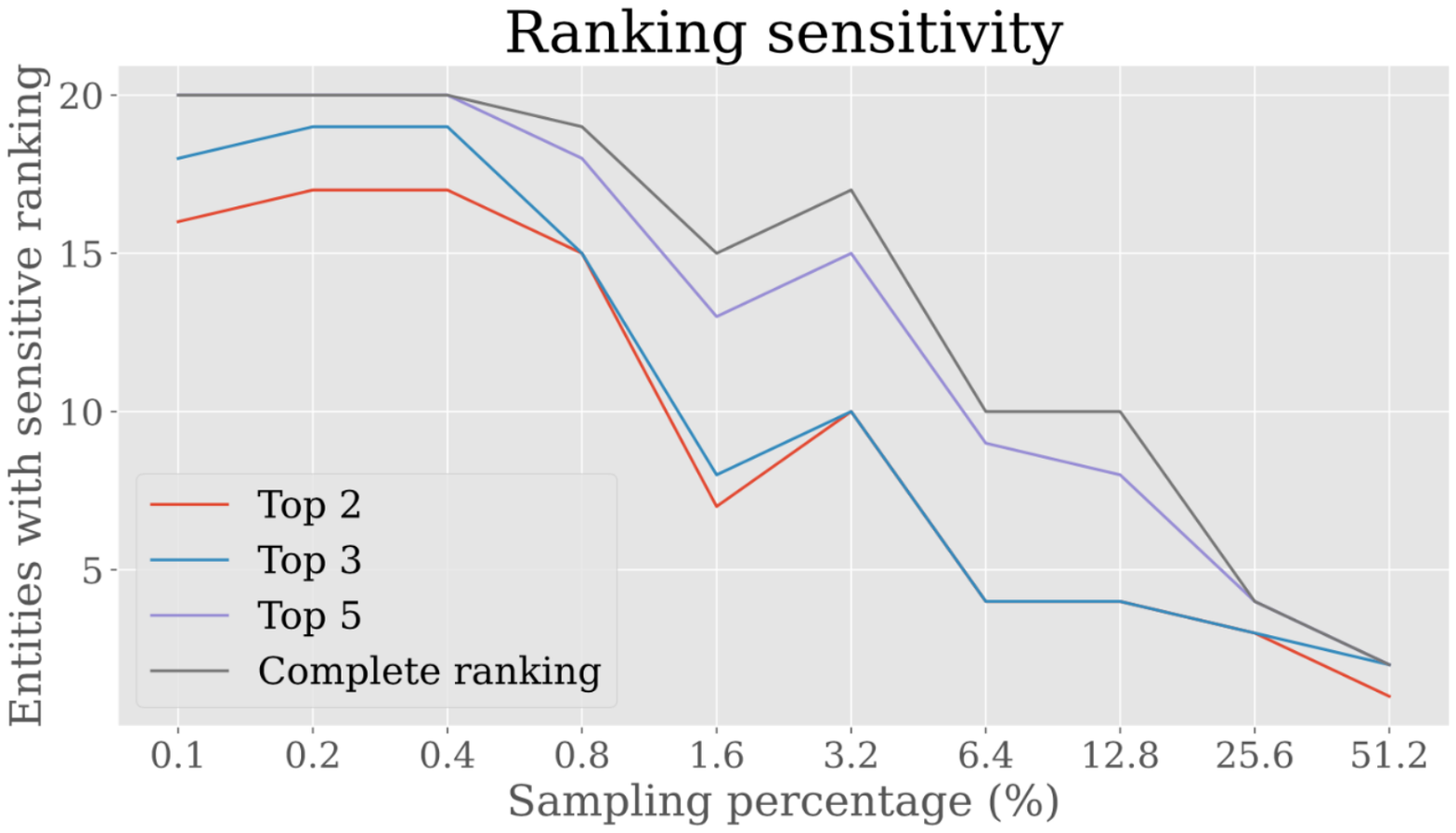}
    \end{center}
    \caption{Number of entities whose feature ranking may vary depending on the chosen distribution inside the uncertainty hyperrectangle, for each sampling percentage. The \texttt{Top $k$} line indicates whether there was a change in the ranking of the top $k$ features, ignoring changes in the rest of the ranking. As expected, sensitivity of the ranking is more common when the sampling percentage is smaller.}
    \label{fig:rankingSensitivity}
    \vspace{0.5cm}
\end{figure}

\vspace{-3mm}
\section{Conclusions}\label{sec:conclusions}

We have analyzed how SHAP scores for classification problems depend on the underlying distribution over the entity population when the distribution varies over a given uncertainty region. As a first and important step, we focused on product distributions and hyperrectangles as uncertainty regions, and obtained algorithmic and complexity results for fundamental problems related to how SHAP scores vary under these conditions. As a proof of concept, we showed through experimentation that considering uncertainty regions has an impact on feature rankings for a non-negligible percentage of the entities, and that the solutions to our proposed problems can provide insight on the relative rankings even in the presence of uncertainty.

\paragraph{Feature Independence.} We stress that from a complexity point of view, it is natural to start with product distributions. As shown in \cite{vandenbroeck-et-al:guy22}, the computation of SHAP scores becomes intractable for trivial classifiers as soon as we consider simple non-product distributions such as naive Bayes distributions. As a consequence, computing maxima of SHAP polynomials is trivially intractable in this setting. As we discussed at the beginning of Section~\ref{sec:complexity-region-max-shap}, we focused on the cases for which computing SHAP scores is tractable, since this gives us the possibility to also obtain tractability for computing maxima of SHAP polynomials. We considered the prominent case of decomposable and deterministic Boolean circuits under product distributions. This case has been shown to be tractable in \cite{vandenbroeck-et-al:guy22} and \cite{arenas-et-al:jmlr23} (in \cite{vandenbroeck-et-al:guy22} product distributions are referred to as fully-factorized distributions). 

For our distributional shift analysis, considering more general distributions on the entity population would still require specifying a class of them and their regions of variation. A natural next step in this direction, that would build on our work, consists in imposing additional {\em domain knowledge} on the product distribution, leading, in particular, to certain dependencies among features. For example, a constraint specifying that ``house lots located at the seaside have a price about \$2M". More generally, a conjunction $\varphi$ of such constraints is associated to an event $E^\varphi \subseteq {\tt ent}(X)$ containing all entities that satisfy it. Conditioning on this event leads to a new distribution $\prob^\prime$ defined by $\prob^\prime(A) := \prob(A|E^\varphi)$, for $A \subseteq {\tt ent}(X)$. 
The shift analysis would be done on the new entity space $\langle {\tt ent}(X), \prob^\prime\rangle$, where (in general) features will not be fully independent anymore. This would be done, without having to start from scratch with $\prob^\prime$, by taking advantage of our previous analysis of the original product distribution $\prob$.

It is worth mentioning that imposing and exploiting {\em domain semantics} when defining and computing
explanation scores is interesting in its own right (see \cite{bertossi2023declarative} in relation to the RESP score). The problem of using domain knowledge in Explainable AI has been scarcely investigated in the literature.

\paragraph{Explanatory Robustness.} 
It would be interesting to know how a local perturbation of a given distribution in the uncertainty region affects the SHAP scores, or their rankings. This would be a proper {\em robustness analysis} with respect to the distribution (as opposed to how SHAP scores vary in a region). Of course, this is a different problem from the most common one of robustness with respect to the perturbation of an input entity (see e.g., \cite{robust,marques-silva}). 

\vspace{2mm}
There are several other problems left open by our work.  
The inclusion of non-binary features and labels is a natural next step. 
It would also be interesting to analyze others proposed scores (e.g., LIME \cite{ribeiro2016should}, RESP \cite{bertossi-et-al:deem}, Kernel-SHAP \cite{lundberg2017unified}) in the setting of distributional uncertainty.



\section*{Acknowledgments}
This research has been supported by an STIC-AmSud program, project AMSUD210006. We appreciate the code made
available by Jorge E. Le\'on (UAI, Chile) for SHAP computation with the same dataset used in our work. \ Bertossi
has been supported by NSERC-DG 2023-04650, the Millennium Institute for Foundational Research on Data (IMFD,
Chile), and  CENIA (Basal ANID FB210017, Chile). \ Romero is funded by the National Center for Artificial Intelligence CENIA FB210017, Basal ANID. Pardal is
funded by DFG VI 1045-1/1. \ Abriola, Cifuentes
and Pardal are funded by FONCyT, ANPCyT and CONICET, in the context of the projects PICT 01481 and PIP
11220200101408CO. Abriola is additionally funded by the project PIBAA 28720210100188CO. \ Martinez is partially funded under the Spanish project PID2022-139835NB-C21 funded by MCIN/AEI/10.13039/501100011033, PIE 20235AT010 and iTrust (PCI2022-135010-2).
Pardal was supported by the DFG grant VI 1045-1/1. 

\appendix

\bibliography{main}

\newpage

\section{Appendix}
\label{sec:appendix}

\subsection{Complete proofs from Section~\ref{sec:prelims-multinlinear}} 
\label{sec:app-prelims-multilinear}

\medskip

{\bf Proposition~\ref{prop:general-maximizesInVertices}.} \emph{
Let $f$ be a multilinear polynomial over $n$ variables. Let $\anInterval\subseteq [0,1]^n$ be a hyperrectangle.
Then the maximum and minimum of $f$ restricted to $\anInterval$ is attained in the vertices of $\anInterval$.}

\medskip

Before showing the proposition, we need an auxiliary result. 

\begin{proposition}
\label{prop:optimumInBoundary}
    Given a closed bounded region $\aRegion \subseteq \R^n$ and a multilinear polynomial $\aMultilinearPolynomial$, the latter attains its maximum and minimum in the boundary $\boundary \aRegion$ of $\aRegion$.\footnote{This proof only requires $\aMultilinearPolynomial$ to be linear in one of its variables for the conclusion to follow. In the case of a multilinear polynomial, this holds for all variables.}
\end{proposition}

\begin{proof}
    Although the proof of this proposition can be found in \cite{laneve-et-al:laneve2010interval}, we show a different argument that we consider simpler.

    Let $\textbf{x}^\star=(x_1^\star, \ldots, x_n^\star) \in R$ be a maximum for $\aMultilinearPolynomial$ over $R$. Since $\aMultilinearPolynomial$ is a multilinear polynomial, we may write it as:

    $$
    \aMultilinearPolynomial(x_1, \ldots, x_n) = x_1 \aMultilinearPolynomial_1(x_2, \ldots, x_n) + \aMultilinearPolynomial_2(x_2, \ldots, x_n)
    $$

    where both $\aMultilinearPolynomial_1$ and $\aMultilinearPolynomial_2$ are multilinear polynomials that do not depend on $x_1$.

    If $\textbf{x}^\star \notin \boundary \aRegion$ then, since $\aRegion$ is bounded, there are reals $\varepsilon_1,\varepsilon_2 > 0$ such that $\textbf{x}^\star + (\varepsilon_1,0,\ldots,0), \textbf{x}^\star - (\varepsilon_2, 0, \ldots, 0) \in \boundary \aRegion$. We show that $\aMultilinearPolynomial_1 (x_2^\star, \ldots, x_n^\star) = 0$: indeed, observe that if $\aMultilinearPolynomial_1(x_2^\star, \ldots, x_n^\star) > 0$, then $\aMultilinearPolynomial(\textbf{x}^\star + (\varepsilon_1, 0,\ldots, 0)) = (x_1^\star + \varepsilon_1) \aMultilinearPolynomial_1(x_2, \ldots, x_n) + \aMultilinearPolynomial_2(x_2, \ldots, x_n) > \aMultilinearPolynomial(\textbf{x}^\star)$ which contradicts $\textbf{x}^\star$ being a maximum. Similarly, but considering $\textbf{x}^\star - (\varepsilon_2,0,\ldots,0)$, it follows that $\aMultilinearPolynomial_1(x_2^\star,\ldots,x_n^\star)$ cannot be negative. Consequently, $\aMultilinearPolynomial_1(x_2^\star, \ldots, x_n^\star) = 0$, which implies that $f(\textbf{x}^\star) = f(\textbf{x}^\star + (\varepsilon_1,0,\ldots,0))$, showing that the polynomial also attains its maximum in $\boundary \aRegion$.

    We can argue analogously to reach the same conclusion for the minimum.
\end{proof}


\begin{proof}[Proof of Proposition~\ref{prop:general-maximizesInVertices}]
Now we are ready to show Proposition~\ref{prop:general-maximizesInVertices}. 
Observe that, if $\aMultilinearPolynomial(x_1,\ldots,x_n)$ is a multilinear polynomial, then $\aMultilinearPolynomial_c(x_2, \ldots, x_n) = \aMultilinearPolynomial(c,x_2,\ldots,x_n)$ is as well a multilinear polynomial for any $c \in \R$. 

Thus, if the bounded region $\aRegion$ is a hyperrectangle $\anInterval\subseteq \mathbb{R}^n$, we can apply Proposition~\ref{prop:optimumInBoundary} recursively $n$ times to conclude that $\aMultilinearPolynomial$ attains its maximum and minimum over the vertices of the hyperrectangle $\anInterval$.
\end{proof}

\medskip

{\bf Theorem~\ref{teo:npcompleteMaximizeMultilinear}.} \emph{
    Given a multilinear polynomial $\aMultilinearPolynomial$, a hyperrectangle $\anInterval = \bigtimes_{i=1}^n [a_i, b_i]$, and a rational $\aBound$, the problem of deciding whether there is an $\textbf{x} \in \anInterval$ such that $\aMultilinearPolynomial(\textbf{x}) \geq \aBound$ is \NPcomplete.}

\begin{proof}
The problem is in \NP, and the proof of this fact was already sketched in Section~\ref{sec:prelims-multinlinear}: due to Proposition~\ref{prop:general-maximizesInVertices}, if $\aMultilinearPolynomial$ attains a value larger than $\aBound$, then there is a vertex $\textbf{x}$ from $\anInterval$ such that $\aMultilinearPolynomial(\textbf{x}) \geq \aBound$. Therefore, a non-deterministic Turing Machine can guess which vertex it is and perform the evaluation.

    The hardness is proved through a reduction from 3-SAT: given a 3-SAT formula $\phi$ on variables $x_1,\ldots, x_n$ and $m$ clauses, we will define a multilinear polynomial $\aMultilinearPolynomial_\phi$ such that $\phi$ is satisfiable if and only if $\aMultilinearPolynomial_\phi(\textbf{x}) \geq 0$ for some $\textbf{x} \in \bigtimes_{i=1}^n [0,1]$.

    Let

    $$
    \phi = \bigwedge_{j=1}^m (l^j_1 \vee l^j_2 \vee l^j_3)
    $$

    where $l^j_i$ are literals, for $i=1,2,3$. We transform each literal into a polynomial by considering the mapping

    $$
    g_{lit}(l) = \begin{cases}
        1 - x_k & l = x_k \\
        x_k & l = \lnot x_k
    \end{cases} 
    $$

    Consequently, we define a mapping for clauses as

    $$
    g_{clause}(l_1 \vee l_2 \vee l_3) = g_{lit}(l_1) g_{lit}(l_2) g_{lit} (l_3)
    $$

    Note that any clause, as long as it does not use the same variable in two of its literals\footnote{This can be assumed without loss of generality.}, is mapped to a multilinear polynomial.
    
    Any assignment of the variables that satisfies a clause $c$ corresponds to an assignment of the variables of $g_{clause}(c)$ that is a root of the polynomial. Similarly, any non-satisfying assignment corresponds to a variable assignment that makes $g_{clause}(c)$ positive.
    
    Finally, we define the multilinear polynomial $\aMultilinearPolynomial_{\phi}$ as

    $$
    \aMultilinearPolynomial_{\phi}(x_1,\ldots,x_n) = - \sum_{j=1}^m g_{clause}(l^j_1 \vee l^j_2 \vee l^j_3).
    $$

    $\aMultilinearPolynomial_\phi$ can be computed in $poly(n+m)$ time given $\phi$, and the correctness of the reduction follows easily from the previous observations about $g_{clause}$.
\end{proof}

\subsection{Complete proofs from Section \ref{sec:complexity-region-max-shap}}

\medskip

{\bf Theorem~\ref{teo:npcomplete_decisiontrees}.}\emph{
    The problem \problemMaxShap{} is \NPhard{} for decomposable and deterministic Boolean circuits. The result holds even when restricted to decision trees.}

\begin{proof}
    We prove the hardness of \problemMaxShap{} through a reduction from \problemVertexCover{}. Let $G=(V,E)$ be a graph and $k\geq 1$ be an integer. We shall define a classifier $M$, an entity $e_0$, a feature $x_0$, an interval $\anInterval$, and a bound $\aBound$ such that $G$ has a vertex cover of size at most $k$ if and only if there is a $\textbf{p} \in \anInterval$ such that $Shap_{M,e_0,x_0}(\textbf{p}) \geq \aBound$. 

    The set of features is $X\defeq \{x_0\}\cup V\cup \{w\}$ where $V$ is the set of nodes of $G$ and $n\defeq |V|$. We consider the \emph{null entity} $e_0$ that assigns $0$ to all features, and the feature $x_0$ to define the SHAP polynomial. The hyperrectangle is $\anInterval = [1,1]\times \bigtimes_{i=1}^n [0,1] \times [\varepsilon, \varepsilon]$. The value of $\varepsilon$ and the bound $\aBound$ will be defined later. Recall that we can assume that the maximum of $Shap_{M,e_0,x_0}$ is attained at a vertex of $\anInterval$, i.e., at a point where $p_{x_0}=1$, $p_{w}=\varepsilon$ and $p_{v} \in \{0,1\}$ for all $v\in V$. 
    The probabilities $(p_v)_{v\in V}$ will be used to ``mark'' which nodes belong to the vertex cover. Formally, for a subset of nodes $C\subseteq V$ of $G$, we associate the vector $\textbf{p}^C=(p_x)_{x\in X}$ such that $p_{x_0}=1$, $p_{w}=\varepsilon$, and for $v\in V$, we have $p_v=0$ if $v\in C$, and $p_v=1$ otherwise. Note that this gives us a bijection between subsets of nodes of $G$ and vertices of $\anInterval$. 

    Intuitively, we devise the classifier $M$ such that $Shap_{M,e_0,x_0}$ has the form 
    \begin{align*}
    Shap_{M,e_0,x_{0}}(\textbf{p}^C) =  - \sum_{\{u,v\} \in E} p_{u}p_{v} I_{u,v}
    - T_{n,\ell} 
    \end{align*}
    where $\ell$ is the size of the subset $C$ and the term $T_{n,\ell}$ grows as $\ell$ grows. The terms $I_{u,v}$ can be made sufficiently large by choosing $\varepsilon$ properly. We choose the bound $\aBound$ to be $T_{n,k}$. If $C$ is a vertex cover, then each term $p_{u}p_{v} I_{u,v}=0$ and hence the first sum of $Shap_{M,e_0,x_{0}}(\textbf{p}^C)$ is $0$. On the other hand, if $C$ is not a vertex cover, we get a small negative value for $Shap_{M,e_0,x_{0}}$, smaller than $\aBound$ in particular, as we get penalized by at least one $I_{u,v}$ (from an uncovered edge $\{u,v\}$). Hence, the only way to obtain $Shap_{M,e_0,x_{0}}(\textbf{p}^C)\geq \aBound$ is to pick $C$ to be, in the first place, a vertex cover, and secondly, one of size $\ell\leq k$.

    Formally, the classifier $M$ is defined as follows. For an entity $e$ over $X$, we have $M(e)=1$ in the following cases (otherwise $M(e)=0$):
    
    \begin{enumerate}
        \item there is an edge $\{u, v\} \in E$ such that $e(x) = 1$ for all $x\in\{x_0, u,v\}$, and $e(x) = 0$ for the remaining features.
        \item $e(x)=1$ for all $x\in \{x_0, w\}$ and $e(x)=0$ for all $x\in V$. 
     \end{enumerate}
    
    By construction, we have $\phi_{M,e_0}(S\cup\{x_0\})=0$ for any $S\subseteq X\setminus\{x_0\}$. Indeed, $e_0(x_0)=0$ but $M(e)=1$ only for entities $e$ with $e(x_0)=1$. On the other hand, for an entity $e'$, the subsets $S\subseteq X\setminus\{x_0\}$ such that $e'\in\consistsWith(e_0,S)$ are precisely the subsets of $Z_{e'}\defeq \{x\in X\setminus \{x_0\}\mid e'(x)=0\}$. Therefore, for a subset $C\subseteq V$ of nodes of $G$, we can write $Shap_{M,e_0,x_0}(\textbf{p}^C)$ as follows:
    \begin{align*}
    & Shap_{M,e_0,x_0}(\textbf{p}^C) \\
    & = \sum_{S\subseteq V\cup\{w\}} c_{|S|} (\phi_{M,e_0}(S\cup\{x_0\}) - \phi_{M,e_0}(S))\\
    & = - \sum_{S\subseteq V\cup\{w\}} c_{|S|} \phi_{M,e_0}(S)\\
    & = - \sum_{S\subseteq V\cup\{w\}} \sum_{e'\in\consistsWith(e_0,S): M(e')=1} c_{|S|} \prob(e'|S)\\
    & = \sum_{e': M(e')=1} \sum_{S\subseteq Z_{e'}} c_{|S|} \prob(e'|S)\\
    & = \sum_{\{u,v\}\in E} p_u p_v I_{u,v} - \varepsilon \sum_{S\subseteq V} c_{|S|} \prod_{y\in V\setminus S}(1-p_y)
    \end{align*}
    where 
    \begin{align*}
    I_{u,v}\defeq \sum_{S\subseteq (V\setminus\{u,v\})\cup\{w\}} c_{|S|} \prod_{y\in ((V\setminus\{u,v\})\cup\{w\})\setminus S} (1-p_y).
    \end{align*}
    
    By the definition of $\textbf{p}^C$, if there is $y\in V\setminus C$ such that $y\notin S$, then $\prod_{y\in V\setminus S}(1-p_y)=0$; otherwise it is $1$. It follows that: 
    \begin{align*}
    Shap_{M,e_0,x_0}(\textbf{p}^C) &= \sum_{\{u,v\}\in E} p_u p_v I_{u,v} - \varepsilon \sum_{S\subseteq V,\  V\setminus C \subseteq S} c_{|S|}\\
    & = \sum_{\{u,v\}\in E} p_u p_v I_{u,v} - \varepsilon \sum_{S\subseteq C} c_{n-|S|}\\
    & = \sum_{\{u,v\}\in E} p_u p_v I_{u,v} - \varepsilon \sum_{i=0}^{\ell}\binom{\ell}{i} c_{n-i}.
    \end{align*}
    Recall $\ell$ denotes the size of $C$. 
    We pick $\varepsilon \defeq \frac{c_{n-1}}{\sum_{i=0}^k \binom{k}{i} c_{n-i}}$ and $\aBound \defeq - \varepsilon \sum_{i=0}^k \binom{k}{i} c_{n-i} = -c_{n-1}$. These values can be computed in $poly(n)$ because $k \leq n$. Similarly, the rest of the reduction can be performed in polynomial time.
    
    Now we show the correctness of the reduction: $G$ has a vertex cover of size at most $k$ if and only if there is a $\textbf{p} \in \anInterval$ such that $Shap_{M,e_0,x_0}(\textbf{p}) \geq \aBound$. Suppose $G$ has a vertex cover of size at most $k$. Then it has a vertex cover $C$ of size precisely $k$. We can take the point $\textbf{p}^C$ since:
    \begin{align*}
    Shap_{M,e_0,x_0}(\textbf{p}^C) & = \sum_{\{u,v\}\in E} p_u p_v I_{u,v} - \varepsilon \sum_{i=0}^{\ell}\binom{\ell}{i} c_{n-i}\\
    & = - \varepsilon \sum_{i=0}^{k}\binom{k}{i} c_{n-i}\\
    & = \aBound.
    \end{align*}
    Suppose now that there is no vertex cover of $G$ of size at most $k$. We claim that $Shap_{M,e_0,x_0}(\textbf{p}) < \aBound$ for all $\textbf{p}\in\anInterval$. It suffices to check this for the vertices of $\anInterval$, that is, for all vectors of the form $\textbf{p}^C$ for subsets of nodes $C\subseteq V$. Now, take an arbitrary subset $C$. Suppose first that  $C$ is not a vertex cover, and that the edge $\{u,v\}$ is not covered, then:
    \begin{align*}
    Shap_{M,e_0,x_0}(\textbf{p}^C) < - p_up_vI_{u,v}  = -I_{u,v}.
    \end{align*}
    Note that $I_{u,v}\geq c_{n-1}$ as $c_{n-1}$ participates in the sum when $S=(V\setminus\{u,v\})\cup\{w\}$. It follows that $Shap_{M,e_0,x_0}(\textbf{p}^C) < -c_{n-1}=\aBound$.
    
    Assume now that $C$ is a vertex cover. Then $Shap_{M,e_0,x_0}(\textbf{p}^C) = -\varepsilon \sum_{i=0}^{\ell}\binom{\ell}{i} c_{n-i}$ where $\ell$ is the size of $C$. We must have $\ell>k$. In particular, $\sum_{i=0}^{\ell}\binom{\ell}{i} c_{n-i} > \sum_{i=0}^{k}\binom{k}{i} c_{n-i}$. We conclude that:
    \begin{align*}
    Shap_{M,e_0,x_0}(\textbf{p}^C) < - \varepsilon \sum_{i=0}^{k}\binom{k}{i} c_{n-i} = \aBound.
    \end{align*}
    
    Finally, we stress that the classifier $M$ can be expressed by a polynomial-sized decision tree $T$: create an accepting path for each entity $e$ satisfying $M(e)=1$. There are $m+1$ such entities, where $m$ is the number of edges in $G$. The tree $T$ is obtained by combining all these accepting paths.
\end{proof}

\subsection{The problem \problemMinShap{}}
\label{sec:app-region-min-shap}

A result analogous to Corollary~\ref{coro:np-region-max-shap} and Theorem~\ref{teo:npcomplete_decisiontrees} can be obtained for the problem \problemMinShap{}.

\begin{theorem}\label{teo:NPcompleteMinShap}

    The problem \problemMinShap{} is in \NP{} for any family $\mathcal{F}$ of models for which it is possible to compute the SHAP score given any underlying product distribution in polynomial time.

    Moreover, it is \NPhard{} for decision trees.
    
\end{theorem}

\begin{proof}
    The \NP{} upper bound follows again by a direct application of Proposition~\ref{prop:general-maximizesInVertices}.

    The hardness follows by a reduction from \problemMaxShap{}. Given $M$, let $1-M$ be the ``negation'' of model $M$:
    \begin{align*}
        (1-M)(e) = 1 - M(e)
    \end{align*}

    Then, for any subset $S\subseteq X$, it holds that
    \begin{align*}
        \phi_{1-M, e}(S) &= \sum_{e' \in \consistsWith(e, S)} \prob(e'|S) (1-M)(e')\\
        &= \sum_{e' \in \consistsWith(e, S)} \prob(e' | S) - \sum_{e' \in \consistsWith(e, S)} \prob(e'|S) M(e')\\
        &= 1 - \phi_{M,e}(S).
    \end{align*}

    And as a consequence,
    \begin{align*}
        &Shap_{1-M, e, x} (\textbf{p})\\
        &= \sum_{S \subseteq X \setminus \{x\}} c_{|S|} \left( \phi_{1-M,e}(S\cup \{x\}) - \phi_{1-M, e}(S)\right)\\
        &= - \sum_{S \subseteq X \setminus \{x\}} c_{|S|} \left( \phi_{M,e}(S \cup \{x\}) - \phi_{M,e}(S) \right)\\
        &= -Shap_{M,e,x}(\textbf{p}).
    \end{align*}

    Thus, $\aBound$ is an upper bound for $Shap_{M,e,x}(\textbf{p})$ for $\textbf{p} \in \anInterval$ if and only if $-\aBound$ is a lower bound of $Shap_{1-M,e,x}(\textbf{p})$ over the same hyperrectangle.
\end{proof}

\subsection{Complete proofs from Section~\ref{sec:complexity-related}}

\medskip

{\bf Theorem~\ref{teo:related_hard}}\emph{
    The problems \problemAmbiguity{} and \problemIrrelevancy{} are \NPhard for decision trees, while \problemDominance{} is \coNPhard.}

\begin{proof}
We start with the hardness of the problem \problemIrrelevancy{}. We give a reduction from \problemVertexCover. Given a graph $G=(V,E)$ and $k\geq 1$ we define a set of features $X$, a model $M$ over features $X$, a feature $x_0$, an entity $e_0$ and an interval $\anInterval$ such that $G$ has a vertex cover of size $k$ if and only if there is a point $\textbf{p} \in \anInterval$ such that $Shap_{M,x_0,e_0}(\textbf{p})=0$.

We define the set of features as 
$$X \defeq \{x_0\}\cup V \cup \{w\}\cup Y \cup Z$$
where $V$ is the set of nodes of $G$, with $n=|V|$, and $Y:=\{y_1,\dots,y_{n-k}\}$, and $Z:=\{z_1,\dots,z_{n-k}\}$.
The hyperrectangle is 
$$\anInterval = [1,1]\bigtimes_{i=1}^n [0,1] \times [\varepsilon, \varepsilon] \bigtimes_{i=1}^{2n-2k} [1,1]$$
where $\varepsilon$ will be defined later on. We consider the \emph{null entity} $e_0$ that assigns $0$ to all features, and the feature $x_0$ to define the SHAP polynomial. 

The idea behind our construction is similar to the one in the proof of Theorem~\ref{teo:npcomplete_decisiontrees}. In this case, we devise a model $M$ such that:
    \begin{align*}
        Shap_{M,e_0,x_0}(\textbf{p}^C) = T_{n,k} - \sum_{\{u, v\} \in E} p_u p_v I_{u,v} - T_{n,\ell}. 
    \end{align*}
    Again, $\ell$ is the size of $C$, and $T_{n,\ell}$ grows as $\ell$ grows. Moreover, $\textbf{p}^C$ denotes the vector associated to the subset of nodes $C\subseteq V$. More precisely, $\textbf{p}^C$ is defined by setting $p_{x_0}=1$, $p_w=\varepsilon$, $p_y=1$ for all $y\in Y$, $p_{z}=1$ for all $z\in Z$, and for $v\in V$, $p_v=0$ if $v\in C$; $p_v=1$ otherwise. In particular, there is a bijection between the subsets of nodes of $G$ and the vertices of $\anInterval$. Note that the first term $T_{n,k}$ does not depend on the set $C$, and consequently $Shap_{M,e_0,x_0}(\textbf{p}^C) = 0$ if $C$ is a vertex cover of size $k$. We need the extra $2(n-k)$ features in $Y\cup Z$ to force $Shap_{M,e_0,x_0}(\textbf{p}^C)$ to have the required form. 

Formally, the classifier $M$ is defined as follows. For an entity $e$ over $X$, we have $M(e)=1$ in the following cases (otherwise $M(e)=0$):
\begin{enumerate}
    \item $e(x)=1$ for all $x\in\{w\}\cup Y$ and $e(x)=0$ for all $x\in \{x_0\}\cup Z$. The values $e(x)$ for $x\in V$ can be arbitrary. \label{itm:1_def_model}
    \item there is an edge $\{u,v\} \in E$ such that $e(x) = 1$ for all $x\in\{x_0, u,v\} \cup Y\cup Z$, and $e(x) = 0$ for the remaining features. \label{itm:2_def_model}
    \item $e(x)=1$ for all $x\in \{x_0, w\}\cup Z$ and $e(x)=0$ for all $x\in V$. The values $e(x)$ for $x\in Y$ can be arbitrary. \label{itm:3_def_model}
 \end{enumerate}

 For a subset $C\subseteq V$, we can write the first term of $Shap_{M,e_0,x_{0}}(\textbf{p}^C)$ as follows:
    \begin{align*}
        &\sum_{S\subseteq V\cup\{w\}\cup Y\cup Z} c_{|S|} \phi_{M,e_0}(S\cup\{x_0\})\\
        & = \sum_{S\subseteq V\cup Z} c_{|S|} \phi_{M,e_0}(S\cup\{x_0\})
    \end{align*}
    since $\phi_{M,e_0}(S\cup\{x_0\})=0$ whenever $S\cap (\{w\}\cup Y)\neq \emptyset$ (this follows by item (\ref{itm:1_def_model}) of the definition of $M$ and the fact that items (\ref{itm:2_def_model}) and (\ref{itm:3_def_model}) only capture entities with $e(x_0)=1$). Moreover, note that if there is $z\in Z$ such that $z\notin S$, then $\phi_{M,e_0}(S\cup\{x_0\})$ has the form $I\cdot (1-p_z)$ for some term $I$, and thus $\phi_{M,e_0}(S\cup\{x_0\})=0$, since $p_z=1$. Hence,
    \begin{align*}
        &\sum_{S\subseteq V\cup Z} c_{|S|} \phi_{M,e_0}(S\cup\{x_0\})\\
        & = \sum_{S\subseteq V} c_{n-k+|S|} \phi_{M,e_0}(S\cup Z\cup\{x_0\}).
    \end{align*}
      
 Note that for $S\subseteq V$, we have:
    \begin{align*}
        \phi_{M,e_0}(S\cup Z\cup\{x_0\}) = p_w\prod_{y\in Y} p_y \sum_{b:V\setminus S\to \{0,1\}} \prod_{x\in V\setminus S}  J_{b,x}
    \end{align*}
    where $J_{b,x} = p_x$ if $b(x)=1$ and $J_{b,x} = 1-p_x$ if $b(x)=0$. Note that $\sum_{b:V\setminus S\to \{0,1\}} \prod_{x\in V\setminus S}  J_{b,x} = 1$, and then we have $\phi_{M,e_0}(S\cup Z\cup\{x_0\}) =  \varepsilon$. 
      It follows that the first term of $Shap_{M,e_0,x_{0}}(\textbf{p}^C)$ is 
            \begin{align*}
  \varepsilon \sum_{i=0}^n \binom{n}{i} c_{n-k+i}. 
      \end{align*}  

 
Denote by $E_1$ and $E_2$ the set of entities defined in items (\ref{itm:2_def_model}) 
and (\ref{itm:3_def_model})
, respectively, of the definition of $M$. Note $E_1\cap E_2=\emptyset$. Recall that, for an entity $e'$, the subsets $S\subseteq X\setminus\{x_0\}$ such that $e'\in\consistsWith(e_0,S)$ are precisely the subsets of $Z_{e'}\defeq \{x\in X\setminus \{x_0\}\mid e'(x)=0\}$. The second term of $Shap_{M,e_0,x_{0}}(\textbf{p}^C)$ can be written as:
    \begin{align*}
        &- \sum_{S\subseteq V\cup\{w\}\cup Y\cup Z} c_{|S|} \phi_{M,e_0}(S)\\
        &= - \sum_{S\subseteq V\cup\{w\}\cup Y\cup Z} \sum_{e': e'(x_0)=1, M(e')=1} c_{|S|} \prob(e'|S)\\
        & = - \sum_{e': e'(x_0)=1, M(e')=1} \sum_{S\subseteq Z_{e'}} c_{|S|} \prob(e'|S)\\
        & = - \sum_{e'\in E_1} \sum_{S\subseteq Z_{e'}} c_{|S|} \prob(e'|S) -  \sum_{e'\in E_2} \sum_{S\subseteq Z_{e'}} c_{|S|} \prob(e'|S) \\
        & = - \sum_{\{u,v\} \in E} p_{u}p_{v} \Bigr[\sum_{S\subseteq V\setminus\{u,v\}} c_{|S|+1} \prod_{x\in V\setminus \{u,v\}\cup S} (1-p_x)\\
         & \qquad\qquad + (1-p_w) \sum_{S\subseteq V\setminus\{u,v\}} c_{|S|} \prod_{x\in V\setminus \{u,v\}\cup S} (1-p_x)\Bigl]\\
         & - \sum_{S\subseteq V\cup Y} c_{|S|} p_w \prod_{z\in Z} p_z \prod_{x\in V\setminus S} (1-p_x) \sum_{b: Y\setminus S\to\{0,1\}} \prod_{y\in Y\setminus S} J_{b,y}
    \end{align*}
      
     where $J_{b,y}$ is defined as above and in the first term of the last equality we separate the sum between those subsets $S$ that include $w$ and those that do not. We have $\sum_{b: Y\setminus S\to\{0,1\}} \prod_{y\in Y\setminus S} J_{b,y} = 1$ and, moreover, $\prod_{x\in V\setminus S} (1-p_x)= 0$ if $V\setminus C\not\subseteq S$. Then $Shap_{M,e_0,x_{0}}(\textbf{p}^C)$ can be written as
    \begin{align*}
        & Shap_{M,e_0,x_{0}}(\textbf{p}^C)\\
        & =   \varepsilon \sum_{i=0}^n \binom{n}{i} c_{n-k+i} - I
   - \varepsilon \sum_{S\subseteq V\cup Y}  c_{|S|} \prod_{x\in V\setminus S} (1-p_x)\\
        &= \varepsilon \sum_{i=0}^n \binom{n}{i} c_{n-k+i} - I
   - \varepsilon \sum_{S\subseteq C\cup Y}  c_{n-\ell + |S|} \\
        &= \varepsilon \sum_{i=0}^n \binom{n}{i} c_{n-k+i} - I
   - \varepsilon \sum_{i=0}^{n-k+\ell}\binom{n-k+\ell}{i} c_{n-\ell + i} \\
    \end{align*}
where $\ell:=|C|$ and         
    \begin{align*}
        & I = \sum_{\{u,v\}\in E} p_{u}p_{v} \Bigr[\sum_{S\subseteq V\setminus\{u,v\}} c_{|S|+1} \prod_{x\in V\setminus \{u,v\}\cup S} \hspace{-0.7em}  (1-p_x)\\
        & \qquad\qquad + (1-\varepsilon) \sum_{S\subseteq V\setminus\{u,v\}} c_{|S|} \prod_{x\in V\setminus \{u,v\}\cup S} \hspace{-0.5em}  (1-p_x)\Bigl].\\
    \end{align*}

   Now we show the correctness of the reduction. Suppose that $C$ is a vertex cover of size $k$. Then:
    \begin{align*}
        &Shap_{M,e_0,x_{0}}(\textbf{p}^C)\\
        & =   \varepsilon \sum_{i=0}^n \binom{n}{i} c_{n-k+i} 
     - \varepsilon \sum_{i=0}^{n-k+k}\binom{n-k+k}{i} c_{n-k + i}\\
        & = 0
    \end{align*}
           
     Suppose there is no vertex cover of size $k$. We claim that $Shap_{M,e_0,x_{0}}(\textbf{p}^C)<0$ for all $C$. Suppose that $C$ is not a vertex cover, and say that the edge $\{u,v\}$ is not covered. Then: 
     
    \begin{align*}
        &Shap_{M,e_0,x_{0}}(\textbf{p}^C)\\
        & \leq   \varepsilon \sum_{i=0}^n \binom{n}{i} c_{n-k+i} - I\\
        &\leq  \varepsilon \sum_{i=0}^n \binom{n}{i} c_{n-k+i} \\
        & \qquad - (1-\varepsilon) \sum_{S\subseteq V\setminus\{u,v\}} c_{|S|} \prod_{x\in V\setminus \{u,v\}\cup S} (1-p_x) \\
        &\leq  \varepsilon \sum_{i=0}^n \binom{n}{i} c_{n-k+i} - (1-\varepsilon) c_{n-2} \\
    \end{align*}

The last inequality holds since $c_{n-2}$ participates in the sum when $S=V\setminus\{u,v\}$. By choosing $\varepsilon$ such that: 
    \begin{align*}
        \frac{1}{\varepsilon} > \frac{\sum_{i=0}^n \binom{n}{i} c_{n-k+i}}{c_{n-2}} +1
    \end{align*}
    we obtain that $Shap_{M,e_0,x_{0}}(\textbf{p}^C) < 0$. Assume now that $C$ is a vertex cover of size $\ell$. We must have $\ell > k$. In this case, we have:
        \begin{align*}
             &Shap_{M,e_0,x_{0}}(\textbf{p}^C)\\
             & =   \varepsilon \sum_{i=0}^n \binom{n}{i} c_{n-k+i} 
     - \varepsilon \sum_{i=0}^{n-k+\ell}\binom{n-k+\ell}{i} c_{n-\ell + i}\\
        \end{align*}
    Note that
        \begin{align*}
             \sum_{i=0}^{n-k+\ell}\binom{n-k+\ell}{i} c_{n-\ell + i}
             & \geq  \sum_{i=\ell-k}^{n-k+\ell}\binom{n-k+\ell}{i} c_{n-\ell + i}\\
             & >  \sum_{i=\ell-k}^{n-k+\ell}\binom{n}{i-(\ell - k)} c_{n-\ell + i}\\
             & = \sum_{j=0}^{n}\binom{n}{j} c_{n-k + j}\\
        \end{align*}
     
Here we use the basic fact $\binom{m}{i}>\binom{m-p}{i-p}$, and the change of variable $j=i-(\ell-k)$. We conclude that $Shap_{M,e_0,x_{0}}(\textbf{p}^C) < 0$. 

Note that $M$ can be computed by a polynomial-sized decision tree: it is the union of three decision trees, one for each case in the definition of $M$. In particular, the decision tree for item (\ref{itm:1_def_model}) ignores all the features in $V$, while the one corresponding to item (\ref{itm:3_def_model}) ignores all the features in $Y$.

\medskip

We can modify the previous reduction to obtain the hardness for \problemAmbiguity{}. 
We simply remove the last features of $Y$ and $Z$, that is, we define $Y:=\{y_1,\dots,y_{n-(k+1)}\}$ and $Z:=\{z_1,\dots,z_{n-(k+1)}\}$, and apply the same reduction. In this case, we have
    \begin{align*}
        & Shap_{M,e_0,x_{0}}(\textbf{p}^C)= \varepsilon \sum_{i=0}^n \binom{n}{i} c_{n-(k+1)+i} - I\\
        &\qquad - \varepsilon \sum_{i=0}^{n-(k+1)+\ell}\binom{n-(k+1)+\ell}{i} c_{n-\ell + i}. \\
    \end{align*}
      
Now, we choose $\varepsilon$ such that:
\begin{align*}
    \frac{1}{\varepsilon} > \frac{\sum_{i=0}^n \binom{n}{i} c_{n-(k+1)+i}}{c_{n-2}} +1.
\end{align*}
           
   If there is no vertex cover of size $k$, then $Shap_{M,e_0,x_{0}}(\textbf{p}^C)\leq 0$ for all $C$. Note that in this case the value could be $=0$, and this may happen for vertex covers $C$ of size $k+1$. Hence, in this case there is no sign change. We claim that, if there is a vertex cover $C$ of size $k$ then there is sign change. First, note that there is always a set $C'$ with value $<0$: it suffices to take any $C'$ that is not a vertex cover (for instance $C'=\emptyset$). We show that $Shap_{M,e_0,x_{0}}(\textbf{p}^C)> 0$. We have that:
    \begin{align*}
         &Shap_{M,e_0,x_{0}}(\textbf{p}^C)\\
         & =   \varepsilon \sum_{i=0}^n \binom{n}{i} c_{n-(k+1)+i} 
     - \varepsilon \sum_{i=0}^{n-1}\binom{n-1}{i} c_{n-k + i}.\\
   \end{align*}
    Then:
    \begin{align*}
        \sum_{i=0}^n \binom{n}{i} c_{n-(k+1)+i} &\geq \sum_{i=1}^n \binom{n}{i} c_{n-(k+1)+i} \\    
        & > \sum_{i=1}^n \binom{n-1}{i-1} c_{n-(k+1)+i} \\
        & = \sum_{j=0}^{n-1} \binom{n-1}{j} c_{n-k+j}. \\
    \end{align*}
    Again, we use $\binom{m}{i}>\binom{m-p}{i-p}$ (in this case for $p=1$), and the change of variable $j=i-1$. We obtain that $Shap_{M,e_0,x_{0}}(\textbf{p}^C) > 0$.

    \medskip

    We conclude by showing the hardness of \problemDominance{}. In particular, 
    we show a reduction from \problemAmbiguity{} to $\overline{\mbox{\problemDominance}}$ (the complement of \problemDominance{}). 

       Let $M$ be a model given by a decision tree over features $X$, $e$ an entity, $x$ a feature, and $\anInterval$ an interval. We will define a new model $M'$, two features $x_1,x_2$ of $M'$, an entity $e'$ and an interval $\anInterval'$ such that there are two points $\textbf{p}_1,\textbf{p}_2 \in \anInterval$ satisfying $Shap_{M,e,x}(\textbf{p}_1) > 0 > Shap_{M,e,x}(\textbf{p}_2)$ if and only if there is a point $\textbf{p}'\in\anInterval'$ such that $Shap_{M',e',x_2}(\textbf{p}') > Shap_{M',e',x_1}(\textbf{p}')$ (that is, $x_1$ does not dominate $x_2$).
    

    The set of features for $M'$ will be $X' = X \cup \{w\}$, with $w$ a feature not in $X$. The features $x_1,x_2$ will be $w,x$ respectively. The entity $e'$ is defined as
    \begin{align*}
        e'(x) = \begin{cases}
            e(x) & x \in X\\
            0 & x = w
        \end{cases}
    \end{align*}

    In order to define the interval for $p_w$, let us consider any interval $I\subseteq [0,1]$, such as $[0,0]$. Indeed, the choice of $I$ is not decisive in this proof. We define $\anInterval' = \anInterval \times I$. 
    
    Finally, the model $M'$ is defined as, given an entity $f$, $M'(f) = M(f|_X)$, where $f|_X$ denotes the restriction of $f$ to the set of features $X$.

    Let $n=|X|$. Note that the values $c_i$ related to the computation of the SHAP score depend on $i$ but also on the number of features $n$. Since the two models $M$ and $M'$ have a different number of features we will write $c_{i,j} = \frac{i! (j-1-i)!}{j!}$ to avoid any ambiguity.

    Since $w$ has no incidence in the prediction of the model, it follows that $\phi_{M',e'}(S\cup \{w\}) = \phi_{M',e'}(S) = \phi_{M,e}(S)$ for any $S \subseteq X$. Then,
    \begin{align*}
        &\phi_{M',e'}(S)\\
        &= \sum_{f \in \consistsWith(e', S)} \prob(f | S) M(f|_X)\\
        &= \sum_{\substack{f \in \consistsWith(e',S) \\ f(w)=0}} \prob(f | S)M(f|_X) + \sum_{\substack{f \in \consistsWith(e',S) \\ f(w)=1}} \prob(f | S) M(f|_X)\\
        &= (1-p_w) \sum_{\substack{f \in \consistsWith(e', S) \\ f(w)=0}} \prob(f|S \cup \{w\}) M(f|_X) \\
        &+ p_w \sum_{\substack{f \in \consistsWith(e', S) \\ f(w)=1}} \prob(f|S \cup \{w\}) M(f|_X)\\
        &= \sum_{f \in \consistsWith(e', S \cup \{w\})} \prob(f|S\cup \{w\}) M(f|_X) = \phi_{M',e'}(S\cup\{w\})
    \end{align*}

    As a consequence, for any $\textbf{p} \in \anInterval'$ it holds that
    \begin{align}\label{eq:null_shap}
        &Shap_{M',e',w}(\textbf{p})\\
        &= \sum_{S \subseteq X' \setminus {w}} c_{|S|,n+1} \left(\phi_{M',e'}(S\cup \{w\}) - \phi_{M',e'}(S)\right)\nonumber \\
        &= 0 \nonumber
    \end{align}

    \newcommand{\phidif}{\mathcal{D}}

    Finally, let $\phidif(S) = \phi_{M',e'}(S\cup\{x\}) - \phi_{M',e'}(S) = \phidif(S \setminus \{w\})$ and observe that: 

    \begin{align}\label{eq:original_shap}
        &Shap_{M',e',x}(\textbf{p}) = \sum_{S \subseteq X' \setminus \{x\}} c_{|S|,n+1} \phidif(S)\nonumber\\
        &= \sum_{\substack{S \subseteq X' \setminus \{x, w\}}} c_{|S|+1,n+1} \phidif(S) + c_{|S|,n+1} \phidif(S) \nonumber \\
        &= \sum_{S \subseteq X' \setminus \{x,w\}} c_{|S|,n} \phidif(S) = Shap_{M,e,x}(\textbf{p})
    \end{align}

    where the third equality holds because $c_{i+1,n+1} + c_{i, n+1} = c_{i, n}$.

    The correctness of the reduction follows directly from Equations~\ref{eq:null_shap} and~\ref{eq:original_shap}, and by observing that the problem \problemAmbiguity{} is hard even when the negative instances are assumed to satisfy $Shap_{M,e,x}(\textbf{p})\leq 0$, for all $\textbf{p}\in\anInterval$ (this follows directly from the previous reduction).
\end{proof}

\end{document}

%% file: macros.tex
\newcommand{\anEntity}{e}
\newcommand{\entities}{\texttt{ent}}

\newcommand{\aModel}{M}

\newcommand{\prob}{\mathbb{P}}
\newcommand{\expectancy}{\mathbb{E}}

\newcommand{\consistsWith}{\texttt{cw}}

\newcommand{\shap}{Shap}

\newcommand{\anInterval}{\mathcal{I}}
\newcommand{\aRegion}{R}

\newcommand{\ignore}[1]{}

\newcommand{\defeq}{\vcentcolon=}

\newcommand{\R}{\mathds{R}}
\newcommand{\boundary}{\partial}

\newcommand{\aBound}{q}

\usepackage{xspace}

\newcommand{\problemMaxShap}{REGION-MAX-SHAP}
\newcommand{\problemMinShap}{REGION-MIN-SHAP}
\newcommand{\problemAmbiguity}{REGION-AMBIGUITY}
\newcommand{\problemIrrelevancy}{REGION-IRRELEVANCY}
\newcommand{\problemDominance}{FEATURE-DOMINANCE}

\newcommand{\problemVertexCover}{VERTEX-COVER}
\newcommand{\NP}{\textsc{NP}\xspace}
\newcommand{\coNP}{\textsc{coNP}\xspace}
\newcommand{\coNPhard}{\textsc{coNP}-hard\xspace}

\newcommand{\NPhard}{\textsc{NP}-hard\xspace}
\newcommand{\NPcomplete}{\textsc{NP}-complete\xspace}
\newcommand{\sharpPhard}{\#\textsc{P}-hard\xspace}

\newcommand{\aMultilinearPolynomial}{f}

\DeclareFontFamily{U}{mathx}{\hyphenchar\font45}
\DeclareFontShape{U}{mathx}{m}{n}{
      <5> <6> <7> <8> <9> <10>
      <10.95> <12> <14.4> <17.28> <20.74> <24.88>
      mathx10
      }{}
\DeclareSymbolFont{mathx}{U}{mathx}{m}{n}
\DeclareMathSymbol{\bigtimes}{1}{mathx}{"91}

\definecolor{darkred}{rgb}{0.55, 0.0, 0.0}
\usepackage[backgroundcolor=orange!20, textcolor={darkred}, textsize=tiny]{todonotes}
\definecolor{green}{RGB}{0,120,0}
\definecolor{hlyellow}{RGB}{250, 250, 190}
\definecolor{edwineditcolor}{RGB}{180,190,165}

\newcommand{\miguel}[1]{\todo[inline,caption={},color=Orange!20, size=\footnotesize]{{\bf Miguel:} #1}}